\documentclass[11pt]{article}
\usepackage{fullpage}
\usepackage{times} 
\usepackage{helvet} 
\usepackage{courier} 
\usepackage[hyphens]{url} 
\usepackage{graphicx} 
\usepackage{graphicx} 
\usepackage[utf8x]{inputenc}
\usepackage{caption}
\usepackage[dvipsnames]{xcolor}
\usepackage{hyperref}
\hypersetup{
	colorlinks=true,
	linkcolor=magenta,     
	citecolor= ForestGreen
}

\usepackage{adjustbox}
\usepackage{times}
\usepackage{amsmath,amsthm,amssymb}

\theoremstyle{remark}

\newtheorem{theorem}{Theorem}

\theoremstyle{remark}
\usepackage{enumerate}
\usepackage{graphicx}
\usepackage{subfigure}
\usepackage{booktabs}
\usepackage{subfiles}
\usepackage{algorithm,algpseudocode}
\usepackage{xspace}
\makeatletter
\DeclareRobustCommand\onedot{\futurelet\@let@token\@onedot}
\def\@onedot{\ifx\@let@token.\else.\null\fi\xspace}

\def\etc{\emph{etc}\onedot} 
 
\def \Real {\mathbb{R}}

\newcommand{\e}{\begin{equation}}
\newcommand{\ee}{\end{equation}}
\newcommand{\en}{\begin{equation*}}
\newcommand{\een}{\end{equation*}}
\newcommand{\eqn}{\begin{eqnarray}}
\newcommand{\eeqn}{\end{eqnarray}}
\newcommand{\bmat}{\begin{bmatrix}}
	\newcommand{\emat}{\end{bmatrix}}

\newcommand{\vct}[1]{\boldsymbol{#1}}
\newcommand{\mtx}[1]{\boldsymbol{#1}}

\def \lg        {\langle}
\def \rg        {\rangle}

\newcommand{\vx}{\vct{x}}
\newcommand{\vy}{\vct{y}}
\newcommand{\vz}{\vct{z}}

\newcommand{\mD}{\mtx{D}}

\newcommand{\mH}{\mtx{H}}
\newcommand{\mJ}{\mtx{J}}

\newcommand{\mL}{\mtx{L}}
\newcommand{\mM}{\mtx{M}}

\newcommand{\mS}{\mtx{S}}

\newcommand{\mU}{\mtx{U}}

\newcommand{\mW}{\mtx{W}}
\newcommand{\mX}{\mtx{X}}

\newcommand{\mZ}{\mtx{Z}}
\newcommand{\mId}{{\bf I}}

\graphicspath{{./figs/}}

\newlength{\imgwidth}
\newboolean{twoColVersion}
\setboolean{twoColVersion}{false}
\newcommand{\twoCol}[2]{\ifthenelse{\boolean{twoColVersion}} {#1} {#2} }

\newcommand{\optr}[1]{\operatorname{\textbf{tr}}(#1)}
\newcommand{\optrb}[1]{\operatorname{\textbf{tr}}\left\{#1\right\}}

\begin{document}
\title{Robust Principal Component Analysis: \\A Construction Error Minimization Perspective}
\author{Kai Liu~\thanks{Computer Science Division, Clemson University. \texttt{kail@clemosn.edu}}
	\\
	\and
	Yarui Cao~\thanks{Computer Science Division, Clemson University. \texttt{yaruic@clemson.edu}}
}
\maketitle
\begin{abstract}
Principal Component Analysis (PCA) has been widely used in data mining and analysis as it can significantly reduce data dimensionality while maintaining the most useful information carried in data. However, from the perspective of minimizing reconstruction error, each data sample’s  error is squared, and therefore sensitive to widely existed outliers and noises which increases dramatically as data dimensionality grows. To alleviate the problem, many researchers focus on improving the robustness of PCA by using more robust norm such as $\ell_{2,p} (p<2)$ or $\ell_1$-norm loss formulation. In this paper we propose a novel optimization framework to systematically solve $\ell_{2,p}$ and $\ell_1$-norm-based PCA problem with rigorous theoretical guarantee, based on which we investigate a very computationally economic updating version. The proposed methods are not only robust to outliers but also easy to implement.
\end{abstract}
\section{Introduction}
In the era of big data, large datasets are increasingly widespread in many areas. To explore these datasets, many machine learning algorithms are proposed to reduce their dimensionality as well as preserve most useful information they carry. Given this purpose, many techniques are developed, and Principal Component Analysis (PCA) \cite{pearson1901liii} is one of the most widely used one. PCA could extract main features and preserve most important information of the data by projecting high-dimensional data into a low-dimensional space, which can save the storage and memory, in addition to speed up the processing process. These principal components (PCs) are eigenvectors of the data's covariance matrix. 
The main idea of PCA is straightforward and can be explained from two perspectives: one is to maximize the variance, and the other one is to minimize the reconstruction error. 

Inspired of maximizing the covariance matrix, many methods are extended from the vanilla PCA. Weighted PCA (WPCA) uses weighted distance to extract features in order to weaken the influence of noise and outliers \cite{fan2011weighted}. Inductive robust principal component analysis (IRPCA) can solve the limitation of RPCA \cite{de2001robust,xu2012robust} with nuclear-norm regularized minimization \cite{bao2012inductive}. Kernel PCA maps original data into a high-dimensional feature space to extract principal components ~\cite{hoffmann2007kernel}. Image principal component analysis (IMPCA) directly constructs the image covariance matrix from the image matrix to perform PCA \cite{yang2002image}. A similar matrix based method, two-dimensional PCA (2DPCA) for example, is proposed to extract features based on 2D image matrices instead of 1D vectors \cite{yang2004two}. Then, Generalized 2DPCA (G2DPCA) overcomes the limitations of 2DPCA \cite{kong2005generalized} and proposes a Bilateral-projection-based 2DPCA (B2DPCA) and a Kernel-based 2DPCA (K2DPCA). Inspired by 2DPCA with $\ell_1$-norm, a robust 2DPCA with non-greedy $\ell_1$-norm maximization is developed to get better performance in image domain \cite{wang2014robust2d}. $\ell_1$-norm-based block principal component analysis (BPCA-L1) improves the robustness of BPCA \cite{wang2012block} and a BPCA with non-greedy $\ell_1$-norm maximization moves a further step to obtain better solutions than BPCA-L1 \cite{li2015block}.

\begin{figure}[htp!]
	\centerline{\includegraphics[width=.5\linewidth,height=7cm]{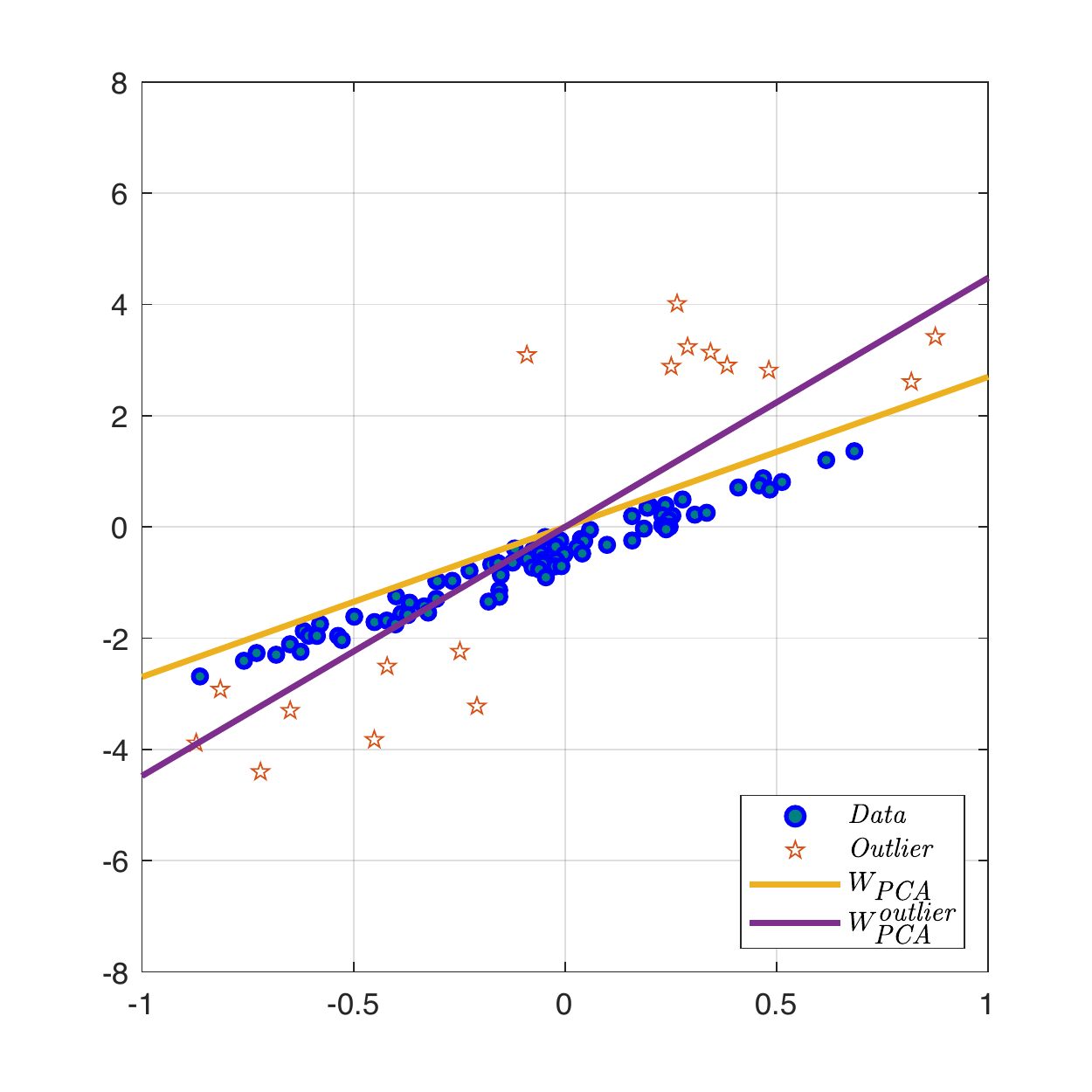}}
	\caption{Projection vectors of vanilla PCA on synthetic data with/without outliers.}
	\label{motivation}
\end{figure}

However, from the perspective of minimizing reconstruction error, as each data's error is squared, therefore it is sensitive to the outliers and noises. As Fig. \ref{motivation} demonstrates, when noise exists in data, which is widely existed in real world, the learned projection matrix may deviate from expected significantly. Considering that outliers and noises are difficult to be identified and eliminated in most cases, many researchers focus on improving the robustness of PCA by utilizing $\ell_{2,p} (0<p<2)$ or $\ell_1$-norm loss functions to alleviate this problem due to their robustness to noise~\cite{yang2019learning,brand2020learning,liu2018high,liu2018learning,liu2019learning,liu2019visual}. $\ell_1$-PCA is proposed to obtain the projection matrix by minimizing the $\ell_1$ norm reconstruction error \cite{ke2005robust}, which is to decompose an image matrix with a weighted combination of the nuclear norm and of the $\ell_1$-norm ~\cite{candes2011robust,zhang2016joint}. Double Robust Principal Component Analysis (DRPCA) extends RPCA by integrating a reconstruction error into its criterion functions to solve out-of-sample problem \cite{wang2020double}. A robust 2DPCA with Frobenius-norm minimization (F-2DPCA) is proposed by \cite{wang2017two} to alleviate the sensitiveness to outliers of 2DPCA by  using $\ell_1$-norm to sum up all different data points, developed in ~\cite{gao2017angle}. To best minimize the reconstruction error while keeping the rotational invariance property, R1-PCA is proposed \cite{ding2006r}. Motivated by that, $\ell_{2,p}$-norm PCA is developed to measure the distance matrix with $\ell_{2,p}$ norm and solve the $\ell_{2,p}$-norm PCA problem iteratively \cite{wang2017ell}. 



\section{Vanilla PCA and Robust PCA via Convex Relaxation}\label{sec:background}

Assume that we have a set of $n$ sample images $\mX = [\vx_1, \vx_2, \cdots, \vx_n]\in \Real^{m\times n}$, where $\vx_i\in \Real^m$ and $i \in[1,n]$ denotes the $i$-th data. Without loss of generality, we here assume that all images in the dataset are centralized. Generally, there are two perspectives of PCA, where the first one is to maximize the variance after project, which can be formulated as:
\begin{equation}\label{eq:variance}
\max	\|\mX^T\mW\|^2_F, \quad  s.t.\quad \mW^T\mW = \mId.
\end{equation}
 The second perspective of PCA is to minimize the reconstruction error and obtain the projection matrix $\mW \in \textbf{R}^{m\times k}$ ($k$ denotes the projected dimension), which is to:
\begin{equation}
	\label{eq:reconstruction}
	\min_{\mW} \|\mX - \mW\mW^T\mX\|_F^2=\sum_{i=1}^{n}\|\vx_i-\mW\mW^T\vx_i\|^2_2, \quad  s.t.\quad \mW^T\mW = \mId,
\end{equation}
where $\mW\mW^T$ is projection matrix and idempotent. The objective remeasure the gap The optimal projection matrix $\mW^*$ then can be obtained via:
\begin{equation}
	\label{W*}
	\mW^* = \arg\min_{\mW} \|\mX - \mW\mW^T\mX\|_F^2, \quad  s.t.\quad \mW^T\mW = \mId.
\end{equation}
By simple algebraic operation: $\|\mX - \mW\mW^T\mX\|_F^2=\optrb{(\mX - \mW\mW^T\mX)(\mX - \mW\mW^T\mX)^T}=\optr{\mX\mX^T-2\mX^T\mW\mW^T\mX+\mX^T\mW\mW^T\mW\mW^T\mX}=\optr{\mX\mX^T-\mX^T\mW\mW^T\mX}$ due to $\mW^T\mW = \mId$, therefore, Eq. (\ref{W*}) is equivalent to:
\begin{equation}
	\label{W^*}
	\mW^* = \arg\max_{\mW} \optr{\mW^T\mX\mX^T\mW}, \quad  s.t.\quad \mW^T\mW = \mId.
\end{equation}
which is equivalent to Eq. (\ref{eq:variance}).

One can see that either in Eq. (\ref{eq:variance}) or Eq. (\ref{eq:reconstruction}), the objective is formulated as squared Frobenius norm, which is known to be sensitive to outliers or noise. Inspired by robust norm and its successful application in various domains such as Nonnegative Matrix Factorization, $K$-means, \etc, there are researches focusing on utilizing robust norm in objective such as $\ell_{2,p} (p<2)$ or $\ell_1$-norm~\cite{nie2011robust,luo2016avoiding,wang2017ell}.

On the other hand, \cite{candes2011robust} proposed a new objective as:
\begin{equation}\label{eq:RPCA}
	\min \|\mL\|_*+\lambda\|\mS\|_1, \quad s.t. \quad \mL+\mS=\mM.
\end{equation}
By observing the difference between Eq. (\ref{eq:RPCA}) and Eq. (\ref{eq:reconstruction}), we can find some connection: 1). $\mL$ in Eq. (\ref{eq:RPCA}) plays a similar role as $\mW\mW^T\mX$ in Eq. (\ref{eq:reconstruction}). The reason is due to $\mW^T\mW = \mId$, then $rank(\mW)=k$, and therefore $rank(\mW\mW^T\mX)\le k$, which means $\mW\mW^T\mX$ has low-rank property. Instead of directly optimizing the rank term in objective, it optimizes the nuclear norm, which is convex. 2). Instead of measuring the loss with squared Frobenius norm, a more robust $\ell_1$ is applied, which is $\|\mS\|_1$ with $\mS=\mX-\mW\mW^T\mX$. Many works follow this formulation by introducing different robust norms such as $\ell_{2,p}$-norm ~\cite{nie2014optimal,shi2018robust}.
\section{Gentle Start}

Before we formally start our algorithm, we first present the following theorems, which will be very useful to establish our convergence analysis.
\begin{theorem}\label{thm:norm1}
	~\cite{liu2015robust} Given $\mX\in\Real^{m\times n}$, $\|\mX\|_1=\optr{\mX\mD\mX^T}$, where $\|\mX\|_1=\sum^n_j\sum^m_i|\mX_{ij}|$ and $\mD$ is diagonal matrix defined by $\mD(i,i) =\frac{\sum^m_{j=1}|\mX_{ji}|}{\|\mX_i\|^2_2}$.
\end{theorem}


\begin{proof}
	
	\begin{equation}
		\begin{aligned}
			\mX=\begin{bmatrix}
				\mX_{11} & \mX_{12} & \cdots & \mX_{1n}\\
				\mX_{21} & \mX_{22} & \cdots & \mX_{2n}\\
				\vdots & \vdots & \ddots & \vdots\\
				\mX_{m1} & \mX_{m2} & \cdots & \mX_{mn}\\
			\end{bmatrix},\ 
			\mX^T=
			\begin{bmatrix}
				\mX_{11} & \mX_{21} & \cdots & \mX_{m1}\\
				\mX_{12} & \mX_{22} & \cdots & \mX_{m2}\\
				\vdots & \vdots & \ddots & \vdots\\
				\mX_{1n} & \mX_{2n} & \cdots & \mX_{mn}\\
			\end{bmatrix}.
		\end{aligned}
	\end{equation}
	By the definition of $\mD$:
	\begin{equation}
		\mD=\begin{bmatrix}
			\frac{\sum^m_{j=1}|\mX_{j1}|}{||\mX_1||^2_2} &  &  & \\
			& \frac{\sum^m_{j=1}|\mX_{j2}|}{||\mX_2||^2_2} &  & \\
			&  & \ddots & \\
			& &  & \frac{\sum^m_{j=1}|\mX_{jn}|}{||\mX_n||^2_2}\\
		\end{bmatrix}.
	\end{equation}
	
	%
	Therefore, we have:
	\begin{equation}
		diag(\mX\mD\mX^T) =
		\begin{bmatrix}
			\frac{\sum^m_{j=1}|\mX_{j1}|}{||\mX_1||^2_2}\mX_{11}^2+ \frac{\sum^m_{j=1}|\mX_{j2}|}{||\mX_2||^2_2}\mX_{12}^2 +  \cdots + \frac{\sum^m_{j=1}|\mX_{jn}|}{||\mX_n||^2_2}\mX_{1n}^2\\
			\frac{\sum^m_{j=1}|\mX_{j1}|}{||\mX_1||^2_2}\mX_{21}^2 + \frac{\sum^m_{j=1}|\mX_{j2}|}{||\mX_2||^2_2}\mX_{22}^2 + \cdots + \frac{\sum^m_{j=1}|\mX_{jn}|}{||\mX_n||^2_2}\mX_{2n}^2\\
			\vdots \\
			\frac{\sum^m_{j=1}|\mX_{j1}|}{||\mX_1||^2_2}\mX_{m1}^2 + \frac{\sum^m_{j=1}|\mX_{j2}|}{||\mX_2||^2_2}\mX_{m2}^2+ \cdots + \frac{\sum^m_{j=1}|\mX_{jn}|}{||\mX_n||^2_2}\mX_{mn}^2\\
		\end{bmatrix}
	\end{equation}
	
	Therefore, we have:
	\begin{equation}
		\optr{\mX\mD\mX^T}
		=|\mX_{11}| + |\mX_{21}| + \cdots + |\mX_{m1}| + \cdots + |\mX_{1n}| + \cdots + |\mX_{mn}|,
	\end{equation}
	
	which apparently is equivalent to $\|\mX\|_1$ and completes the proof.
\end{proof}
\begin{theorem}\label{thm:1}
	Denote $\mJ=\|\mX-\mW\mW^T\mX\|_2=\sum_{i=1}^{n}\|\vx_i-\mW\mW^T\vx_i\|_2$, then  $\nabla\mJ(\mW)=-\mX\mD\mX^T\mW$, where $\mD$ is diagonal matrix with $\mD(i,i)=\frac{1}{\|\vx_i-\mW\mW^T\vx_i\|}$.
\end{theorem}
\begin{proof}
	\begin{equation}
		\begin{aligned}
			&\frac{\nabla\|\vx_i-\mW\mW^T\vx_i\|_2}{\nabla\mW}\\
			=&\frac{\nabla\sqrt{\optrb{(\vx_i-\mW\mW^T\vx_i)^T(\vx_i-\mW\mW^T\vx_i)}}}{\nabla\mW}\\
			=&\frac{1}{2\sqrt{\optrb{(\vx_i-\mW\mW^T\vx_i)^T(\vx_i-\mW\mW^T\vx_i)}}}\frac{\nabla\optrb{(\vx_i-\mW\mW^T\vx_i)^T(\vx_i-\mW\mW^T\vx_i)}}{\nabla\mW}\\
			=&\frac{\nabla\optr{\vx_i^T\vx_i-\vx_i^T\mW\mW^T\vx_i}}{\nabla\mW}\frac{1}{2\|\vx_i-\mW\mW^T\vx_i\|}\\
			=&\frac{-2\vx_i\vx_i^T\mW}{2\|\vx_i-\mW\mW^T\vx_i\|}\\
			=&\frac{-\vx_i\vx_i^T}{\|\vx_i-\mW\mW^T\vx_i\|}\mW\\
			=&-\mD(i,i)\vx_i\vx_i^T\mW.
		\end{aligned}
	\end{equation}
	Therefore, $\nabla\mJ(\mW)=\frac{\nabla\sum_{i=1}^{n}\|\vx_i-\mW\mW^T\vx_i\|_2}{\nabla\mW}=\frac{\sum_{i=1}^{n}\nabla\|\vx_i-\mW\mW^T\vx_i\|_2}{\nabla\mW}=-\mX\mD\mX^T\mW$.
\end{proof}
\begin{theorem}\label{thm:p}
	Denote $\mJ=\|\mX-\mW\mW^T\mX\|_2^p=\sum_{i=1}^{n}\|\vx_i-\mW\mW^T\vx_i\|_2^p$, then  $\nabla\mJ(\mW)=-\mX\mD\mX^T\mW$, where $\mD$ is diagonal matrix with $\mD(i,i)=p\cdot\|\vx_i-\mW\mW^T\vx_i\|^{p-2}$.
\end{theorem}
\begin{proof}
	\begin{equation}
		\begin{aligned}
			&\frac{\nabla\|\vx_i-\mW\mW^T\vx_i\|_2^p}{\nabla\mW}\\
			=&p\cdot\|\vx_i-\mW\mW^T\vx_i\|_2^{p-1}\frac{\nabla\|\vx_i-\mW\mW^T\vx_i\|_2}{\nabla\mW}\\
			=&-p\cdot\|\vx_i-\mW\mW^T\vx_i\|_2^{p-2}\vx_i\vx_i^T\mW\\
			=&-\mD(i,i)\vx_i\vx_i^T\mW,
		\end{aligned}
	\end{equation}
	therefore $\nabla\mJ(\mW)=\frac{\nabla\sum_{i=1}^{n}\|\vx_i-\mW\mW^T\vx_i\|_2^p}{\nabla\mW}=\frac{\sum_{i=1}^{n}\nabla\|\vx_i-\mW\mW^T\vx_i\|_2^p}{\nabla\mW}=-\mX\mD\mX^T\mW$.
\end{proof}
One can see that when $p=1$, Theorem \ref{thm:p} degenerates into Theorem \ref{thm:1}. While $p=2$, it is in accordance with vanilla PCA. 
\begin{theorem}\label{thm:lip_cont_grad}
	~\cite{nesterov2018lectures} For a function $f$ with a Lipschitz continuous gradient $L$, if $\|\nabla f(\vx)-\nabla f(\vy)\|\le L\|\vx-\vy\|$, then $f(\vy)\le f(\vx)+\nabla f(\vx)^T(\vy-\vx)+\frac{L}{2}\|\vy-\vx\|^2.$
\end{theorem}
\begin{proof}
	\begin{equation}
		\begin{aligned}
			f(\vy)&=f(\vx)+\int_{0}^{1}\lg\nabla f(\vx+\tau(\vy-\vx)),\vy-\vx\rg d\tau\\
			&=f(\vx)+\lg \nabla f(\vx), \vy-\vx\rg + \int_{0}^{1}\lg\nabla f(\vx+\tau(\vy-\vx))-\nabla f(\vx),\vy-\vx\rg d\tau,
		\end{aligned}
	\end{equation}
	therefore, we have
	\begin{equation}
		\begin{aligned}
			&|f(\vy)-f(\vx)-\lg \nabla f(\vx),\vy-\vx\rg|\\
			=&\left |\int_{0}^{1}\lg\nabla f(\vx+\tau(\vy-\vx)),\vy-\vx\rg d\tau\right |\\
			\le&\int_{0}^{1}|\lg\nabla f(\vx+\tau(\vy-\vx))-\nabla f(\vx),\vy-\vx\rg |d\tau\\
			\le&\int_{0}^{1}\|\nabla f(\vx+\tau(\vy-\vx))-\nabla f(\vx)\|\cdot\|\vy-\vx\|d\tau\\
			\le&\int_{0}^{1}\tau L\|\vy-\vx\|^2d\tau\\
			=&\frac{L}{2}\|\vy-\vx\|^2,
		\end{aligned}
	\end{equation}
	which completes the proof.
\end{proof}

\section{$\ell_1$-norm Robust PCA}\label{sec:l1}
In this section, we study $\ell_1$-norm robust PCA problem from the perspective of minimizing the construction error. In light of squared Frobenius norm leading sensitive solution to noise, we use $\ell_1$-norm measuring the different between original data and its projection:
\begin{equation}\label{eq:l1}
	\sum_{i=1}^{n}\|\vx_i-\mW\mW^T\vx_i\|_1=\|\mX-\mW\mW^T\mX\|_1, \quad  s.t.\quad \mW^T\mW = \mId.
\end{equation}

\subsection{Monotonically Decreasing Algorithm for $\ell_1$-norm Robust PCA}
In this subsection, we focus on optimizing
\begin{equation}
	\min_{\mW}\ \mJ=\|\mX-\mW\mW^T\mX\|_1, \quad  s.t.\quad \mW^T\mW = \mId.
\end{equation}
According to Theorem~\ref{thm:norm1}: 
\begin{equation}\begin{aligned}
\mJ&=\|\mX-\mW\mW^T\mX\|_1\\&=\optrb{(\mX-\mW\mW^T\mX)\mD(\mX-\mW\mW^T\mX)^T}\\&= \optr{\mX\mD\mX^T - \mW^T\mX\mD\mX^T\mW}.
\end{aligned}\end{equation}
Therefore $\nabla \mJ(\mW)=-2\mX\mD\mX^T\mW, \nabla_{\mW}^2 \mJ=-2\mX\mD\mX^T$. 

According to Taylor expansion:
\begin{equation}\label{eq:taylor}
	\begin{aligned}
		\mJ(\mW^+)-\mJ(\mW)&=\lg \mW^+-\mW,\nabla \mJ(\mW)\rg+\frac{1}{2}\lg \mW^+-\mW,\nabla_{\tilde{\mW}}^2 \mJ(\mW^+-\mW)\rg\\
		&\le \lg \mW^+-\mW,\nabla\mJ(\mW)\rg+\frac{L}{2}\| \mW^+-\mW\|^2_F,
	\end{aligned}
\end{equation}
where $L:=2\|\mX\mD\mX^T\|_2$.

For sake of convergence analysis, we reformulate Eq. (\ref{eq:l1}) as: 
\begin{equation}\label{eq:obj}
	\begin{split}
		\min_\textbf{W} \mJ(\textbf{W})+Q(\textbf{W}),\quad Q(\textbf{W})=
		\begin{cases}
			0, \textbf{W}^T\textbf{W}=\textbf{I}\\
			\infty, else.
		\end{cases}
	\end{split}
\end{equation}

Denote $\textbf{W}^+$ as the updated $\textbf{W}$ in the next step as
\begin{equation}
	\label{w+}
	\begin{split}
		\textbf{W}^+ = \text{argmin}_{\Bar{\textbf{W}}} \langle\nabla \mJ(\textbf{W}),\Bar{\textbf{W}}-\textbf{W} \rangle +\frac{1}{2t}||\Bar{\textbf{W}}-\textbf{W}||^2_F + Q(\Bar{\textbf{W}}).
	\end{split}
\end{equation}

Let $P(\Bar{\textbf{W}},\textbf{W})$ denote $\langle\nabla \mJ(\textbf{W}),\Bar{\textbf{W}}-\textbf{W} \rangle + \frac{1}{2t}||\Bar{\textbf{W}}-\textbf{W}||^2_F$. Then from the definition of $\textbf{W}^+$, we have
\begin{equation}
	\label{p_w}
	P(\textbf{W}^+,\textbf{W})\leq P(\textbf{W},\textbf{W}).
\end{equation}

Combining Eq. (\ref{w+}) and Eq. (\ref{p_w}), we have
\begin{equation}
	\label{Q_w}
	\begin{split}
		\langle\nabla \mJ(\textbf{W}),\textbf{W}^+-\textbf{W} \rangle + \frac{1}{2t}||\textbf{W}^+-\textbf{W}||^2_F + Q(\textbf{W}^+)\leq Q(\textbf{W}).
	\end{split}
\end{equation}
Then we obtain:
\begin{equation}
	\label{J_wQ_w}
	\begin{split}
		&\mJ(\textbf{W}) + Q(\textbf{W}) - \mJ(\textbf{W}^+) - Q(\textbf{W}^+)\\
		\geq &\mJ(\textbf{W}) - \mJ(\textbf{W}^+) + \langle\nabla \mJ(\textbf{W}),\textbf{W}^+-\textbf{W} \rangle + \frac{1}{2t}||\textbf{W}^+-\textbf{W}||^2_F \\
		\geq &-\langle\nabla \mJ(\textbf{W}),\textbf{W}^+-\textbf{W} \rangle - \frac{L}{2} ||\textbf{W}^+-\textbf{W}||^2_F  + \langle\nabla \mJ(\textbf{W}),\textbf{W}^+-\textbf{W} \rangle + \frac{1}{2t}||\textbf{W}^+-\textbf{W}||^2_F \\
		= &\frac{1}{2}(\frac{1}{t}-L)||\textbf{W}^+-\textbf{W}||^2_F, 
	\end{split}
\end{equation}
where the second line comes from Eq. (\ref{Q_w}) and third line from Eq. (\ref{eq:taylor}). The above equation tells that if we set $t\le\frac{1}{L}$, then the objective in Eq. (\ref{eq:obj}) will monotonically decrease~\cite{zhu2018dropping,bolte2014proximal}. 

We now turn to find the optimal solution for $\mW^+$. According to Eq. (\ref{w+}), we have
\begin{equation}
	\begin{split}
		\mW^+ &= \text{argmin}_{\Bar{\textbf{W}}}~ \langle\nabla \mJ(\textbf{W}),\Bar{\textbf{W}}-\textbf{W} \rangle+\frac{1}{2t}||\Bar{\textbf{W}}-\textbf{W}||^2_F + Q(\Bar{\textbf{W}})\\
		&= \text{argmin}_{\Bar{\textbf{W}}}~ \frac{1}{2t}||\Bar{\textbf{W}}-\textbf{W} + t\nabla \mJ(\textbf{W})||^2_F + Q(\Bar{\textbf{W}})\\
		&= \text{argmin}_{\Bar{\textbf{W}}^T\Bar{\textbf{W}}=\textbf{I}}~ \frac{1}{2t}||\Bar{\textbf{W}}-(\textbf{W}-t\nabla \mJ(\textbf{W}))||^2_F\\
		&= \text{argmin}_{\Bar{\textbf{W}}^T\Bar{\textbf{W}}=\textbf{I}}~ \frac{1}{2t}||\Bar{\textbf{W}}-(\textbf{W}+t\cdot2\textbf{XDX}^T\textbf{W})||^2_F.
	\end{split}    
\end{equation}

Let $\textbf{R}$ denote $\textbf{W}+t\cdot2\textbf{XDX}^T\textbf{W}$, then we have
\begin{equation}
	\begin{split}
		\textbf{W}^+ &= \text{argmin}_{\Bar{\textbf{W}}^T\Bar{\textbf{W}}=\textbf{I}}~ \frac{1}{2t}||\Bar{\textbf{W}}-\textbf{R}||^2_F
		= \text{argmiax}_{\Bar{\textbf{W}}^T\Bar{\textbf{W}}=\textbf{I}}~ \optr{\Bar{\textbf{W}}^T\textbf{R}}.
	\end{split}
\end{equation}

As a result~\cite{liu2019spherical}, if $[\textbf{U},\textbf{S},\textbf{V}]=SVD(\textbf{R})$, then $\Bar{\textbf{W}}=\textbf{UV}^T$ is the optimal solution.

\begin{algorithm}[tb]
	\caption{$\ell_1$-norm PCA}
	\label{alg:algorithm}
	\textbf{Input}: 
	$\textbf{X} = [\textbf{x}_1, \textbf{x}_2, \cdots, \textbf{x}_n]\in \textbf{R}^{m\times n}$, where $\textbf{X}$ is centralized;\\
	Initialized $\textbf{W}\in \textbf{R}^{m\times k}$ which satisfies $\textbf{W}^T\textbf{W} = \textbf{I}$.\\
	\textbf{Output}: $\textbf{W}$
	\begin{algorithmic}[1] 
		\While{not converge}
		\State Calculate diagonal matrix $\textbf{D}(i,i) = \frac{\sum^m_{j=1}|\textbf{Y}_{ji}|}{||\textbf{Y}_i||^2_2}$, where $\textbf{Y} = \textbf{X} - \textbf{WW}^T\textbf{X}$.
		\State $\textbf{R}=\textbf{W}+t\cdot2\textbf{XDX}^T\textbf{W}$, where $t = \frac{1}{||2\textbf{XDX}^T||_2}$. (or $\textbf{R}=\textbf{W}+t\cdot\textbf{XDX}^T\textbf{W}$, where $t = \frac{1}{||\textbf{XDX}^T||_2}$).
		\State $[\textbf{U},\textbf{S},\textbf{V}] = SVD(\textbf{R})$.
		\State $\textbf{W} = \textbf{UV}^T$ .
		\EndWhile
		\State \textbf{return} $\textbf{W}$
	\end{algorithmic}
\end{algorithm}
\subsection{A  faster version for optimization}
Inspired by \cite{beck2009fast,guan2012nenmf} which used Nesterov Accelerated Gradient (NAG) to accelerate process, we also utilize momentum for our algorithm, and we found in practice, the momentum based algorithm is always faster to converge. It is worth noting that different from \cite{beck2009fast,guan2012nenmf} where the objective is convex while the constraint is convex as well, in our objective, neither the objective nor the constraint is convex, therefore, no theoretical guarantee can be drawn at this stage, but empirical results demonstrate that it  always plays as an efficient accelerator. 
\begin{algorithm}[tb]
	\caption{Faster $\ell_1$-norm PCA}
	\label{alg:algorithm2}
	\textbf{Input}: 
	$\textbf{X} = [\textbf{x}_1, \textbf{x}_2, \cdots, \textbf{x}_n]\in \textbf{R}^{m\times n}$, where $\textbf{X}$ is centralized;\\
	Initialized $\textbf{W}\in \textbf{R}^{m\times k}$ which satisfies $\textbf{W}^T\textbf{W} = \textbf{I}$ and $s=1$.\\
	\textbf{Output}: $\textbf{W}$
	\begin{algorithmic}[1] 
		\While{not converge}
		\State Calculate diagonal matrix $\textbf{D}(i,i) = \frac{\sum^m_{j=1}|\textbf{Y}_{ji}|}{||\textbf{Y}_i||^2_2}$, where $\textbf{Y} = \textbf{X} - \textbf{WW}^T\textbf{X}$.
		\State $\textbf{V}=\textbf{W}+\frac{s-2}{s+1}(\textbf{W}-\textbf{W}_{old})$
		\State $\textbf{R}=\textbf{V}+t\cdot\textbf{XDX}^T\textbf{V}$, where $t = \frac{1}{\|\textbf{XDX}^T\|_2}$.
		\State $[\textbf{U},\textbf{S},\textbf{V}] = SVD(\textbf{R})$.
		\State $\textbf{W} = \textbf{UV}^T$ .
		\State $s = s+1$.
		\EndWhile
		\State \textbf{return} $\textbf{W}$
	\end{algorithmic}
\end{algorithm}
\begin{algorithm}[h!]
	\caption{A very Fast $\ell_1$-norm PCA}
	\label{alg:algorithm1}
	\textbf{Input}: 
	$\textbf{X} = [\textbf{x}_1, \textbf{x}_2, \cdots, \textbf{x}_n]\in \textbf{R}^{m\times n}$, where $\textbf{X}$ is centralized;\\
	Initialized $\textbf{W}\in \textbf{R}^{m\times k}$ which satisfies $\textbf{W}^T\textbf{W} = \textbf{I}$.\\
	\textbf{Output}: $\textbf{W}$
	\begin{algorithmic}[1] 
		\While{not converge}
		\State Calculate diagonal matrix $\textbf{D}(i,i) = \frac{\sum^m_{j=1}|\textbf{Y}_{ji}|}{||\textbf{Y}_i||^2_2}$, where $\textbf{Y} = \textbf{X} - \textbf{WW}^T\textbf{X}$.
		\State $[\textbf{U},\textbf{S},\textbf{U}] = SVD(\textbf{XDX}^T)$.
		\State $\textbf{W} = \textbf{U}(:,1:r)$ .
		\EndWhile
		\State \textbf{return} $\textbf{W}$
	\end{algorithmic}
\end{algorithm}

\subsection{A Very fast version for optimization}
Recall Eq. (\ref{eq:l1}) that $\mJ=tr(\textbf{XDX}^T - \textbf{W}^T\textbf{XDX}^T\textbf{W})$, as we are going to minimize $\mJ$, that is equal to maximize $tr(\textbf{W}^T\textbf{XDX}^T\textbf{W})$, given the constraint of $\textbf{W}^T\textbf{W}=\textbf{I}$, and symmetry of $\textbf{XDX}^T$ (one can even prove that it is positive definite), the optimal $\textbf{W}$ can be obtained by SVD: $[\mU,\mS,\mU]=svd(\textbf{XDX}^T)$, then $\textbf{W}=\mU(:,1:r)$. It can be summarized in Algorithm \ref{alg:algorithm1}. It is worth noting that the faster version may suffer from non-monotonically decreasing theoretically, however in most cases it is reliable and smoothing as Fig. (\ref{fig:l1pca}) illustrates.

We conclude this part by pointing out that though Algorithm~\ref{alg:algorithm} has rigorous theoretical guarantee while the rest two not, but in practice, it still suffers from slow convergence as within each iteration, an SVD is executed with complexity $\mathcal{O}(m^3)$. The loop is considerably long as it is fundamentally a gradient descent method which requires many iterations to converge.

\begin{figure}[h!]
	\centering
	\includegraphics[width=.6\linewidth]{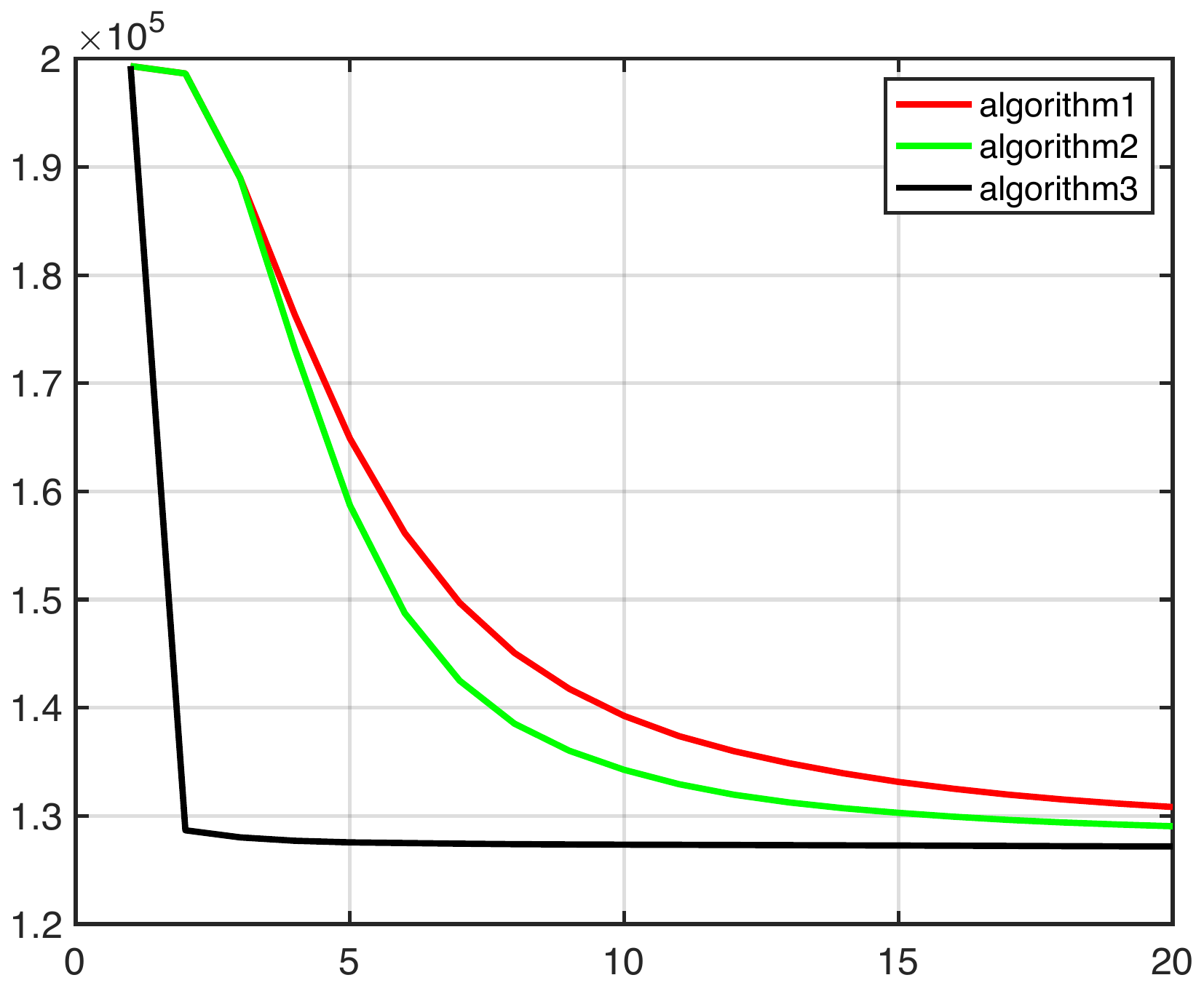}
	\caption{Objective changes with updates.}
	\label{fig:l1pca}
\end{figure}

\section{$\ell_{2,p}$-norm Robust PCA}
\begin{algorithm}[tb]
	\caption{$\ell_{2,p}$-norm PCA}
	\label{alg:algorithm3}
	\textbf{Input}: 
	$\textbf{X} = [\textbf{x}_1, \textbf{x}_2, \cdots, \textbf{x}_n]\in \textbf{R}^{m\times n}$, where $\textbf{X}$ is centralized;\\
	Initialized $\textbf{W}\in \textbf{R}^{m\times k}$ which satisfies $\textbf{W}^T\textbf{W} = \textbf{I}$.\\
	\textbf{Output}: $\textbf{W}$
	\begin{algorithmic}[1] 
		\While{not converge}
		\State Calculate diagonal matrix $\mD(i,i)=p\cdot\|\vx_i-\mW\mW^T\vx_i\|^{p-2}$.
		\State $\textbf{R}=\textbf{W}+t\cdot\textbf{XDX}^T\textbf{W}$, where $t = \frac{1}{||\textbf{XDX}^T||_2}$.
		\State $[\textbf{U},\textbf{S},\textbf{V}] = SVD(\textbf{R})$.
		\State $\textbf{W} = \textbf{UV}^T$ .
		\EndWhile
		\State \textbf{return} $\textbf{W}$
	\end{algorithmic}
\end{algorithm}

\begin{algorithm}[tb]
	\caption{Faster $\ell_{2,p}$-norm PCA}
	\label{alg:algorithm5}
	\textbf{Input}: 
	$\textbf{X} = [\textbf{x}_1, \textbf{x}_2, \cdots, \textbf{x}_n]\in \textbf{R}^{m\times n}$, where $\textbf{X}$ is centralized;\\
	Initialized $\textbf{W}\in \textbf{R}^{m\times k}$ which satisfies $\textbf{W}^T\textbf{W} = \textbf{I}$ and $s=1$.\\
	\textbf{Output}: $\textbf{W}$
	\begin{algorithmic}[1] 
		\While{not converge}
		\State Calculate diagonal matrix $\mD(i,i)=p\cdot\|\vx_i-\mW\mW^T\vx_i\|^{p-2}$.
		\State $\textbf{V}=\textbf{W}+\frac{s-2}{s+1}(\textbf{W}-\textbf{W}_{old})$
		\State $\textbf{R}=\textbf{V}+t\cdot\textbf{XDX}^T\textbf{V}$, where $t = \frac{1}{\|\textbf{XDX}^T\|_2}$.
		\State $[\textbf{U},\textbf{S},\textbf{V}] = SVD(\textbf{R})$.
		\State $\textbf{W} = \textbf{UV}^T$ .
		\State $s = s+1$.
		\EndWhile
		\State \textbf{return} $\textbf{W}$
	\end{algorithmic}
\end{algorithm}

\begin{algorithm}[h!]
	\caption{A very Fast $\ell_{2,p}$-norm PCA}
	\label{alg:algorithm4}
	\textbf{Input}: 
	$\textbf{X} = [\textbf{x}_1, \textbf{x}_2, \cdots, \textbf{x}_n]\in \textbf{R}^{m\times n}$, where $\textbf{X}$ is centralized;\\
	Initialized $\textbf{W}\in \textbf{R}^{m\times k}$ which satisfies $\textbf{W}^T\textbf{W} = \textbf{I}$.\\
	\textbf{Output}: $\textbf{W}$
	\begin{algorithmic}[1] 
		\While{not converge}
		\State Calculate diagonal matrix $\mD(i,i)=p\cdot\|\vx_i-\mW\mW^T\vx_i\|^{p-2}$.
		\State $[\textbf{U},\textbf{S},\textbf{U}] = SVD(\textbf{XDX}^T)$.
		\State $\textbf{W} = \textbf{U}(:,1:r)$ .
		\EndWhile
		\State \textbf{return} $\textbf{W}$
	\end{algorithmic}
\end{algorithm}
In this section, we study $\ell_{2,p}$-norm robust PCA problem from the perspective of minimizing the construction error. In light of squared Frobenius norm leading sensitive solution to noise, we use $\ell_{2,p}$-norm measuring the different between original data and its projection:
\begin{equation}\label{eq:l2p}
	\sum_{i=1}^{n}\|\vx_i-\mW\mW^T\vx_i\|_2^p=\|\mX-\mW\mW^T\mX\|_2^p, \quad  s.t.\quad \mW^T\mW = \mId.
\end{equation}
where we define $\|\mZ\|_2^p=\sum_{i}\|\vz_i\|^p_2$, where $\vz_i$ denotes the $i$-th column of $\mZ$.
\subsection{Monotonically Decreasing Algorithm for $\ell_{2,p}$-norm Robust PCA}
Similar to Section~\ref{sec:l1}, we can propose the updating algorithm via:
\begin{equation}
	\begin{split}
		\textbf{W}^+ = \text{argmin}_{\Bar{\textbf{W}}} \langle\nabla \mJ(\textbf{W}),\Bar{\textbf{W}}-\textbf{W} \rangle +\frac{1}{2t}||\Bar{\textbf{W}}-\textbf{W}||^2_F + Q(\Bar{\textbf{W}})
	\end{split}
\end{equation}
where $\nabla \mJ(\textbf{W})=-\mX\mD\mX^T\mW$ is from Theorem~\ref{thm:p}, also one can set $t=\frac{1}{L}$ where $L=\|\mX\mD\mX^T\|_2$. 

The whole process to prove its convergence is almost the same as the above Section, therefore we skip the details but put more focus on the accelerating version.
\subsection{Accelerated Algorithm for $\ell_{2,p}$-norm Robust PCA}
If we denote $\mH=\nabla^2\mJ(\mW)=-\mX\mD\mX^T$, then $\nabla \mJ(\mW)=\mH\mW$.
\begin{equation}\label{eq:fastest}\begin{split}
		&\min_{\mW^+} \mJ(\mW^+)=\min_{\mW^+} \mJ(\mW^+)-\mJ(\mW)\\\approx&\min_{\mW^+}\langle \nabla \mJ(\mW), \mW^+-\mW\rangle + \frac{1}{2} \langle \mW^+-\mW, \mH(\mW^+-\mW)\rangle \\=&\min_{\mW^+}\frac{1}{2}\langle \mW^+-\mW, \mH(\mW^+-\mW)+2\mH\mW\rangle
		\\=&\min_{\mW^+}\frac{1}{2}\langle \mW^+-\mW, \mH(\mW^++\mW)\rangle
		\\=&\min_{\mW^+}\frac{1}{2}\optr{(\mW^+-\mW)^T\mH(\mW^++\mW)}
		\\=&\min_{\mW^+}\frac{1}{2}\optr{\mH(\mW^+\mW^{+T}+\mW\mW^{+T}-\mW^+\mW^T)}
		\\=&\min_{\mW^+}\frac{1}{2}\optr{\mW^{+T}\mH\mW^+}
		\\=&\max_{\mW^+}\frac{1}{2}\optr{\mW^{+T}\mX\mD\mX^T\mW^+}
	\end{split}
\end{equation}
Apparently, if $[\mU,\mS,\mU]=svd(\mX\mD\mX^T)$, then $\mW^+=\mU(:,1:r)$. Apparently the two different robust norms PCA algorithms in Alg. ~\ref{alg:algorithm} v.s. Alg. ~\ref{alg:algorithm3}, Alg. ~\ref{alg:algorithm2} v.s. Alg. ~\ref{alg:algorithm5}, and Alg. ~\ref{alg:algorithm1} v.s. Alg. ~\ref{alg:algorithm4} they are almost identical except the difference in calculating diagonal matrix $\mD$.
\bibliography{main}
\bibliographystyle{ieeetr}
\end{document}